\documentclass[twoside,11pt]{article}

\usepackage{jair}
\usepackage{theapa}

\usepackage{amsfonts}
\usepackage{amssymb}
\usepackage{amsmath}
\usepackage{graphics}
\usepackage{rotating}
\usepackage{url}

\newtheorem{theorem}{Theorem}

\newenvironment{proof}
{\noindent \textsc{Proof.}} {\hfill \begin{small}$\square$ \end{small}
\vspace{0.2cm}}

\newenvironment{proofsketch}
{\noindent \textsc{Sketch of Proof.}} {\hfill \begin{small}$\square$
\end{small} \vspace{0.2cm}}

\newcommand{\onecell}[3]{
\parbox[c]{2cm}{
    \vspace{0.1cm}
    \raggedleft #1\\
    \vspace{0.1cm}
    \raggedright #3}
}

\newtheorem{proposition}{Proposition}

 \jairheading{29}{2007}{391-419}{1/2007}{8/2007}
 \ShortHeadings{Obtaining Reliable Feedback for Sanctioning Reputation Mechanisms}
 {Jurca \& Faltings}
 \firstpageno{391}

\begin{document}

%\title{Truthful Reputation Information in Electronic Markets without Independent Verification}

%\title{Obtaining Reliable Feedback when Clients can Commit to Report Honestly}

\title{Obtaining Reliable Feedback for Sanctioning Reputation Mechanisms}

\author{\name Radu Jurca \email radu.jurca@epfl.ch \\
        \name Boi Faltings \email boi.faltings@epfl.ch \\
       \addr Ecole Polytechnique F\'ed\'erale de Lausanne (EPFL) \\
        Artificial Intelligence Laboratory (LIA) \\
        CH-1015 Lausanne, Switzerland\\
        \url{http://liawww.epfl.ch}
}

% For research notes, remove the comment character in the line below.
% \researchnote
\maketitle

\begin{abstract}
Reputation mechanisms offer an effective alternative to verification
authorities for building trust in electronic markets with moral hazard. Future
clients guide their business decisions by considering the feedback from past
transactions; if truthfully exposed, cheating behavior is sanctioned and thus
becomes irrational.

It therefore becomes important to ensure that rational clients have the right
incentives to report honestly. As an alternative to side-payment schemes that
explicitly reward truthful reports, we show that honesty can emerge as a
rational behavior when clients have a repeated presence in the market. To this
end we describe a mechanism that supports an equilibrium where truthful
feedback is obtained. Then we characterize the set of pareto-optimal equilibria
of the mechanism, and derive an upper bound on the percentage of false reports
that can be recorded by the mechanism. An important role in the existence of
this bound is played by the fact that rational clients can establish a
reputation for reporting honestly.
\end{abstract}

\section{Introduction}
The availability of ubiquitous communication through the Internet is driving
the migration of business transactions from direct contact between people to
electronically mediated interactions. People interact electronically either
through human-computer interfaces or through programs representing humans,
so-called agents. In either case, no physical interactions among entities
occur, and the systems are much more susceptible to fraud and deception.

Traditional methods to avoid cheating involve cryptographic schemes and
\emph{trusted third parties} (TTP's) that overlook every transaction. Such
systems are very costly, introduce potential bottlenecks, and may be difficult
to deploy due to the complexity and heterogeneity of the environment: e.g.,
agents in different geographical locations may be subject to different
legislation, or different interaction protocols.

Reputation mechanisms offer a novel and effective way of ensuring the necessary
level of trust which is essential to the functioning of any market. They are
based on the observation that agent strategies change when we consider that
interactions are repeated: the other party remembers past cheating, and changes
its terms of business accordingly in the future. Therefore, the expected gains
due to future transactions in which the agent has a higher reputation can
offset the loss incurred by not cheating in the present. This effect can be
amplified considerably when such reputation information is shared among a large
population, and thus multiplies the expected future gains made accessible by
honest behavior.

Existing reputation mechanisms enjoy huge success. Systems such as
eBay\footnote{www.ebay.com} or Amazon\footnote{www.amazon.com} implement
reputation mechanisms which are partly credited for the businesses' success.
Studies show that human users seriously take into account the reputation of the
seller when placing bids in online auctions \cite{Houser/Wooders:2006}, and
that despite the incentive to free ride, feedback is provided in more than half
of the transactions on eBay \cite{RZ:2002}.

One important challenge associated with designing reputation mechanisms is to
ensure that truthful feedback is obtained about the actual interactions, a
property called \textit{incentive-compatibility}. Rational users can regard the
private information they have observed as a valuable asset, not to be freely
shared. Worse even, agents can have external incentives to misreport and thus
manipulate the reputation information available to other agents
\cite{Harmon:2004}. Without proper measures, the reputation mechanism will
obtain unreliable information, biased by the strategic interests of the
reporters.

Honest reporting incentives should be addressed differently depending on the
predominant role of the reputation mechanisms. The \emph{signaling} role is
useful in environments where the service offered by different providers may
have different quality, but all clients interacting with the same provider are
treated equally (markets with \emph{adverse selection}). This is the case, for
example, in a market of web-services. Different providers possess different
hardware resources and employ different algorithms; this makes certain
web-services better than others. Nevertheless, all requests issued to the same
web-service are treated by the same program. Some clients might experience
worse service than others, but these differences are random, and not determined
by the provider. The feedback from previous clients statistically estimates the
quality delivered by a provider in the future, and hence signals to future
clients which provider should be selected.

The \emph{sanctioning} role, on the other hand, is present in settings where
service requests issued by clients must be individually addressed by the
provider. Think of a barber, who must skillfully shave every client that walks
in his shop. The problem here is that providers must exert care (and costly
effort) for satisfying every service request. Good quality can result only when
enough effort was exerted, but the provider is better off by exerting less
effort: e.g., clients will anyway pay for the shave, so the barber is better
off by doing a sloppy job as fast as possible in order to have time for more
customers. This \emph{moral hazard} situation can be eliminated by a reputation
mechanism that punishes providers for not exerting effort. Low effort results
in negative feedback that decreases the reputation, and hence the future
business opportunities of the provider. The future loss due to a bad reputation
offsets the momentary gain obtained by cheating, and makes cooperative behavior
profitable.

%
%An example of such an environment is a market of delivery services. Every
%delivery requires costly effort, as the delivery destination must be scheduled
%on one of the delivery routes. Despite the provider's best intention, packages
%do sometimes get lost, damaged or delayed. The moral hazard problem of the
%service provider is whether to continue exerting care when handling every one
%of the submitted packages, or to save some costs by intentionally
%\emph{dropping} every now and then some packages.

There are well known solutions for providing honest reporting incentives for
signaling reputation mechanisms. Since all clients interacting with a service
receive the same quality (in a statistical sense), a client's private
observation influences her belief regarding the experience of other clients. In
the web-services market mentioned before, the fact that one client had a bad
experience with a certain web-service makes her more likely to believe that
other clients will also encounter problems with that same web-service. This
correlation between the client's private belief and the feedback reported by
other clients can be used to design feedback payments that make honesty a Nash
equilibrium. When submitting feedback, clients get paid an amount that depends
both on the the value they reported and on the reports submitted by other
clients. As long as others report truthfully, the expected payment of every
client is maximized by the honest report -- thus the equilibrium.
\citeA{MRZ:2005} and \citeA{JF_EC:2006} show that incentive-compatible payments
can be designed to offset both reporting costs and lying incentives.

For \emph{sanctioning} reputation mechanisms the same payment schemes are not
guaranteed to be incentive-compatible. Different clients may experience
different service quality because the provider decided to exert different
effort levels. The private beliefs of the reporter may no longer be correlated
to the feedback of other clients, and therefore, the statistical properties
exploited by \citeA{MRZ:2005} are no longer present.

As an alternative, we propose different incentives to motivate honest reporting
based on the repeated presence of the client in the market. Game theoretic
results (i.e., the \emph{folk theorems}) show that repeated interactions
support new equilibria where present deviations are made unattractive by future
penalties. Even without a reputation mechanism, a client can guide her future
play depending on the experience of previous interactions. As a first result of
this paper, we describe a mechanism that indeed supports a cooperative
equilibrium where providers exert effort all the time. The reputation mechanism
correctly records when the client received low quality.

There are certainly some applications where clients repeatedly interact with
the same seller with a potential moral hazard problem. The barber shop
mentioned above is one example, as most people prefer going to the same barber
(or hairdresser). Another example is a market of delivery services. Every
package must be scheduled for timely delivery, and this involves a cost for the
provider. Some of this cost may be saved by occasionally dropping a package,
hence the moral hazard. Moreover, business clients typically rely on the same
carrier to dispatch their documents or merchandise. As their own business
depends on the quality and timeliness of the delivery, they do have the
incentive to form a lasting relationship and get good service. Yet another
example is that of a business person who repeatedly travels to an offshore
client. The business person has a direct interest to repeatedly obtain good
service from the hotel which is closest to the client's offices.

We assume that the quality observed by the clients is also influenced by
environmental factors outside the control of, however observable by, the
provider. Despite the barber's best effort, a sudden movement of the client can
always generate an accidental cut that will make the client unhappy. Likewise,
the delivery company may occasionally lose or damage some packages due to
transportation accidents. Nevertheless, the delivery company (like the barber)
eventually learns with certainty about any delays, damages or losses that
entitle clients to complain about unsatisfactory service.

The mechanism we propose is quite simple. Before asking feedback from the
client, the mechanism gives the provider the opportunity to acknowledge
failure, and reimburse the client. Only when the provider claims good service
does the reputation mechanism record the feedback of the client. Contradictory
reports (the provider claims good service, but the client submits negative
feedback) may only appear when one of the parties is lying, and therefore, both
the client and the provider are sanctioned: the provider suffers a loss as a
consequence of the negative report, while the client is given a small fine.

One equilibrium of the mechanism is when providers always do their best to
deliver the promised quality, and truthfully acknowledge the failures caused by
the environmental factors. Their ``honest'' behavior is motivated by the threat
that any mistake will drive the unsatisfied client away from the market. When
future transactions generate sufficient revenue, the provider does not afford
to risk losing a client, hence the equilibrium.

Unfortunately, this socially desired equilibrium is not unique. Clients can
occasionally accept bad service and keep returning to the same provider because
they don't have better alternatives. Moreover, since complaining for bad
service is sanctioned by the reputation mechanism, clients might be reluctant
to report negative feedback. Penalties for negative reports and the clients'
lack of choice drives the provider to occasionally cheat in order to increase
his revenue.

As a second result, we characterize the set of pareto-optimal equilibria of our
mechanism and prove that the amount of unreported cheating that can occur is
limited by two factors. The first factor limits the amount of cheating in
general, and is given by the quality of the alternatives available to the
clients. Better alternatives increase the expectations of the clients,
therefore the provider must cheat less in order to keep his customers.

The second factor limits the amount of unreported cheating, and represents the
cost incurred by clients to establish a reputation for reporting the truth. By
stubbornly exposing bad service when it happens, despite the fine imposed by
the reputation mechanism, the client signals to the provider that she is
committed to always report the truth. Such signals will eventually change the
strategy of the provider to full cooperation, who will avoid the punishment for
negative feedback. Having a reputation for reporting truthfully is of course,
valuable to the client; therefore, a rational client accepts to lie (and give
up the reputation) only when the cost of building a reputation for reporting
honestly is greater than the occasional loss created by tolerated cheating.
This cost is given by the ease with which the provider switches to cooperative
play, and by the magnitude of the fine imposed for negative feedback.

Concretely, this paper proceeds as follows. In Section \ref{related_work} we
describe related work, followed by a more detailed description of our setting
in Section \ref{setting}. Section \ref{GTanalysis} presents a game theoretic
model of our mechanism and an analysis of reporting incentives and equilibria.
Here we establish the existence of the cooperative equilibrium, and derive un
upper bound on the amount of cheating that can occur in any pareto-optimal
equilibrium.

In Section \ref{buidingReputation} we establish the cost of building a
reputation for reporting honestly, and hence compute an upper bound on the
percentage of false reports recorded by the reputation mechanism in any
equilibrium.

We continue in Section \ref{evilBuyers} by analyzing the impact of malicious
buyers that explicitly try to destroy the reputation of the provider. We give
some initial approximations on the worst case damage such buyers can cause to
providers. Further discussions, open issues and directions for future work are
discussed in Section \ref{future_work}. Finally, Section \ref{conclusions}
concludes our work.

\section{Related Work}
\label{related_work}

The notion of \textit{reputation} is often used in Game Theory to signal the
commitment of a player towards a fixed strategy. This is what we mean by saying
that \textit{clients establish a reputation for reporting the truth}: they
commit to always report the truth. Building a reputation usually requires some
incomplete information repeated game, and can significantly impact the set of
equilibrium points of the game. This is commonly referred to as the
\textit{reputation effect}, first characterized by the seminal papers of
\citeA{KMRW:1982}, \citeA{Kreps/Wilson:1982} and \citeA{Milgrom/Roberts:1982}.

The reputation effect can be extended to all games where a player ($A$) could
benefit from committing to a certain strategy $\sigma$ that is not credible in
a complete information game: e.g., a monopolist seller would like to commit to
fight all potential entrants in a chain-store game \cite{Selten:1978}, however,
this commitment is not credible due to the cost of fighting. In an incomplete
information game where the commitment type has positive probability, $A$'s
opponent ($B$) can at some point become convinced that $A$ is playing as if she
were the commitment type. At that point, $B$ will play a best response against
$\sigma$, which gives $A$ the desired payoff. Establishing a reputation for the
commitment strategy requires time and cost. When the higher future payoffs
offset the cost of building reputation, the reputation effect prescribes
minimum payoffs any equilibrium strategy should give to player $A$ (otherwise,
$A$ can profitably deviate by playing as if she were a commitment type).

\citeA{Fudenberg/Levine:1989} study the class of all repeated games in which a
long-run player faces a sequence of single-shot opponents who can observe all
previous games. If the long-run player is sufficiently patient and the
single-shot players have a positive prior belief that the long-run player might
be a commitment type, the authors derive a lower bound on the payoff received
by the long-run player in any Nash equilibrium of the repeated game. This
result holds for both finitely and infinitely repeated games, and is robust
against further perturbations of the information structure (i.e., it is
independent of what other types have positive probability).

\citeA{Schmidt:1993} provides a generalization of the above result for the two
long-run player case in a special class of games called of ``conflicting
interests'', when one of the players is sufficiently more patient than the
opponent. A game is of conflicting interests when the commitment strategy of
one player ($A$) holds the opponent ($B$) to his minimax payoff. The author
derives an upper limit on the number of rounds $B$ will not play a best
response to $A$'s commitment type, which in turn generates a lower bound on
$A$'s equilibrium payoff. For a detailed treatment of the reputation effect,
the reader is directed to the work of \citeA{Mailath/Samuelson:2006}.

In computer science and information systems research, \textit{reputation}
information defines some aggregate of feedback reports about past transactions.
This is the semantics we are using when referring to the reputation of the
provider. Reputation information encompasses a unitary appreciation of the
personal attributes of the provider, and influences the trusting decisions of
clients. Depending on the environment, reputation has two main roles: to
\emph{signal} the capabilities of the provider, and to \emph{sanction} cheating
behavior \cite{Kuwabara:2003}.

\emph{Signaling} reputation mechanisms allow clients to learn which providers
are the most capable of providing good service. Such systems have been widely
used in computational trust mechanisms. \citeA{Birk:2001} and \citeA{Bis:2000}
describe systems where agents use their direct past experience to recognize
trustworthy partners. The global efficiency of the market is clearly increased,
however, the time needed to build the reputation information prohibits the use
of this kind of mechanisms in a large scale online market.

A number of signaling reputation mechanisms also take into consideration
indirect reputation information, i.e., information reported by peers.
\citeA{Sch:2000} and \citeA{Yu:2002,Yu/Singh:2003} use social networks in order
to obtain the reputation of an unknown agent. Agents ask acquaintances several
hops away about the trustworthiness of an unknown agent. Recommendations are
afterwards aggregated into a single measure of the agent's reputation. This
class of mechanisms, however intuitive, does not provide any rational
participation incentives for the agents. Moreover, there is little protection
against untruthful reporting, and no guarantee that the mechanism cannot be
manipulated by a malicious provider in order to obtain higher payoffs.

Truthful reporting incentives for signaling reputation mechanisms are described
by \citeA{MRZ:2005}. Honest reports are explicitly rewarded by payments that
take into account the value of the submitted report, and the value of a report
submitted by another client (called the \textit{reference reporter}). The
payment schemes are designed based on \textit{proper scoring rules},
mathematical functions that make possible the revelation of private beliefs
\cite{Cooke:1991}. The essence behind honest reporting incentives is the
observation that the private information a client obtains from interacting with
a provider changes her belief regarding the reports of other clients. This
change in beliefs can be exploited to make honesty an ex-ante Nash equilibrium
strategy.

\citeA{JF_EC:2006} extend the above result by taking a computational approach
to designing incentive compatible payment schemes. Instead of using closed form
scoring rules, they compute the payments using an optimization problem that
minimizes the total budget required to reward the reporters. By also using
several reference reports and filtering mechanisms, they render the payment
mechanisms cheaper and more practical.

\citeA{Dellarocas:2005} presents a comprehensive investigation of binary
\textit{sanctioning} reputation mechanisms. As in our setting, providers are
equally capable of providing high quality, however, doing so requires costly
effort. The role of the reputation mechanism is to encourage cooperative
behavior by punishing cheating: negative feedback reduces future revenues
either by excluding the provider from the market, or by decreasing the price
the provider can charge in future transactions. Dellarocas shows that simple
information structures and decision rules can lead to efficient equilibria,
given that clients report honestly.

Our paper builds upon such mechanisms by addressing reporting incentives. We
will abstract away the details of the underlying reputation mechanism through
an explicit penalty associated with a negative feedback. Given that such high
enough penalties exist, any reputation mechanism (i.e., feedback aggregation
and trusting decision rules) can be plugged in our scheme.

In the same group of work that addresses reporting incentives, we mention the
work of \citeA{Braynov/Sandholm:2002}, \citeA{Dellarocas:2002_LNCS2531} and
\citeA{Papaioannou/Stamoulis:2005}. \citeauthor{Braynov/Sandholm:2002} consider
exchanges of goods for money and prove that a market in which agents are
trusted to the degree they deserve to be trusted is equally efficient as a
market with complete trustworthiness. By scaling the amount of the traded
product, the authors prove that it is possible to make it rational for sellers
to truthfully declare their trustworthiness. Truthful declaration of one's
trustworthiness eliminates the need of reputation mechanisms and significantly
reduces the cost of trust management. However, the assumptions made about the
trading environment (i.e. the form of the cost function and the selling price
which is supposed to be smaller than the marginal cost) are not common in most
electronic markets.

For e-Bay-like auctions, the Goodwill Hunting mechanism
\cite{Dellarocas:2002_LNCS2531} provides a way to make sellers indifferent
between lying or truthfully declaring the quality of the good offered for sale.
Momentary gains or losses obtained from misrepresenting the good's quality are
later compensated by the mechanism which has the power to modify the
announcement of the seller.

\citeA{Papaioannou/Stamoulis:2005} describe an incentive-compatible reputation
mechanism that is particularly suited for peer-to-peer applications. Their
mechanism is similar to ours, in the sense that both the provider and the
client are punished for submitting conflicting reports. The authors
experimentally show that a class of common lying strategies are successfully
deterred by their scheme. Unlike their results, our paper considers \emph{all}
possible equilibrium strategies and sets bounds on the amount of untruthful
information recorded by the reputation mechanism.

\section{The Setting}
\label{setting}

We assume an online market, where rational clients (she) repeatedly request the
same service from one provider (he). Every client repeatedly interacts with the
service provider, however, successive requests from the same client are always
interleaved with enough requests generated by other clients. Transactions are
assumed sequential, the provider does not have capacity constraints, and
accepts all requests.

The price of service is $p$ monetary units, and the service can have either
high ($q_1$) or low ($q_0$) quality. Only high quality is valuable to the
clients, and has utility $u(q_1)=u$. Low quality has utility 0, and can be
precisely distinguished from high quality. Before each round, the client can
decide to request the service from the provider, or quit the market and resort
to an outside provider that is completely trustworthy. The outside provider
always delivers high quality service, but for a higher price $p(1+\rho)$.

If the client decides to interact with the online provider, she issues a
request to the provider, and pays for the service. The provider can now decide
to exert low ($e_0$) or high ($e_1$) effort when treating the request. Low
effort has a normalized cost of 0, but generates only low quality. High effort
is expensive (normalized cost equals $c(e_1)=c$) and generates high quality
with probability $\alpha < 1$. $\alpha$ is fixed, and depends on the
environmental factors outside the control of the provider. $\alpha p > c$, so
that it is individually rational for providers to exert effort.

After exerting effort, the provider can observe the quality of the resulting
service. He can then decide to deliver the service as it is, or to acknowledge
failure and roll back the transaction by fully reimbursing\footnote{In reality,
the provider might also pay a penalty for rolling back the transaction. As long
as this penalty is small, the qualitative results we present in this paper
remain valid.} the client. We assume perfect delivery channels, such that the
client perceives exactly the same quality as the provider. After delivery, the
client inspects the quality of service, and can accuse low quality by
submitting a negative report to the reputation mechanism.

The reputation mechanism (RM) is unique in the market, and trusted by all
participants. It can oversee monetary transactions (i.e., payments made between
clients and the provider) and can impose fines on all parties. However, the RM
does not observe the effort level exerted by the provider, nor does it know the
quality of the delivered service.

The RM asks feedback from the client only if she chose to transact with the
provider in the current round (i.e., paid the price of service to the provider)
and the provider delivered the service (i.e., provider did not reimburse the
client). When the client submits negative feedback, the RM punishes both the
client and the provider: the client must pay a fine $\varepsilon$, and the
provider accumulates a negative reputation report.

\subsection{Examples}
\label{setting_example}

Although simplistic, this model retains the main characteristics of several
interesting applications. A delivery service for perishable goods (goods that
lose value past a certain deadline) is one of them. Pizza, for example, must be
delivered within 30 minutes, otherwise it gets cold and loses its taste. Hungry
clients can order at home, or drive to a more expensive local restaurant, where
they're sure to get a hot pizza. The price of a home delivered pizza is $p=1$,
while at the restaurant, the same pizza would cost $p(1+\rho) = 1.2$. In both
cases, the utility of a warm meal is $u = 2$.

The pizza delivery provider must exert costly effort to deliver orders within
the deadline. A courier must be dispatched immediately (high effort), for an
estimated cost of $c=0.8$. While such action usually results in good service
(the probability of a timely delivery is $\alpha = 99\%$), traffic conditions
and unexpected accidents (e.g., the address is not easily found) may still
delay some deliveries past the deadline.

Once at the destination, the delivery person, as well as the client, know if
the delivery was late or not. As it is common practice, the provider can
acknowledge being late, and reimburse the client. Clients may provide feedback
to a reputation mechanism, but their feedback counts only if they were not
reimbursed. The client's fine for submitting a negative report can be set for
example at $\varepsilon = 0.01$. The future loss to the provider caused by the
negative report (and quantified through $\bar{\varepsilon}$) depends on the
reputation mechanism.

A simplified market of car garagists or plumbers could fit the same model. The
provider is commissioned to repair a car (respectively the plumbing) and the
quality of the work depends on the exerted effort. High effort is more costly
but ensures a lasting result with high probability. Low effort is cheap, but
the resulting fix is only temporary. In both cases, however, the warranty
convention may specify the right of the client to ask for a reimbursement if
problems reoccur within the warranty period. Reputation feedback may be
submitted at the end of the warranty period, and is accepted only if
reimbursements didn't occur.

An interesting emerging application comes with a new generation of web services
that can optimally decide how to treat every request. For some service types, a
high quality response requires the exclusive use of costly resources. For
example, computation jobs require CPU time, storage requests need disk space,
information requests need queries to databases. Sufficient resources, is a
prerequisite, but not a guarantee for good service. Software and hardware
failures may occur, however, these failures are properly signaled to the
provider. Once monetary incentives become sufficiently important in such
markets, intelligent providers will identify the moral hazard problem, and may
act strategically as identified in our model.

\section{Behavior and Reporting Incentives}
\label{GTanalysis}

From game theoretic point of view, one interaction between the client and the
provider can be modeled by the extensive-form game ($G$) with imperfect public
information, shown in Figure \ref{fig:game}. The client moves first and decides
(at node 1) whether to play $in$ and interact with the provider, or to play
$out$ and resort to the trusted outside option.

Once the client plays $in$, the provider can chose at node 2 whether to exert
high or low effort (i.e., plays $e_1$ or $e_0$ respectively). When the provider
plays $e_0$ the generated quality is low. When the provider plays $e_1$, nature
chooses between high quality ($q_1$) with probability $\alpha$, and low quality
($q_0$) with probability $1-\alpha$. The constant $\alpha$ is assumed common
knowledge in the market. Having seen the resulting quality, the provider
delivers (i.e., plays $d$) the service, or acknowledges low quality and rolls
back the transaction (i.e., plays $l$) by fully reimbursing the client. If the
service is delivered, the client can report positive ($1$) or negative ($0$)
feedback.

\begin{figure}[t]
    \centerline{\includegraphics[width=0.9\columnwidth]{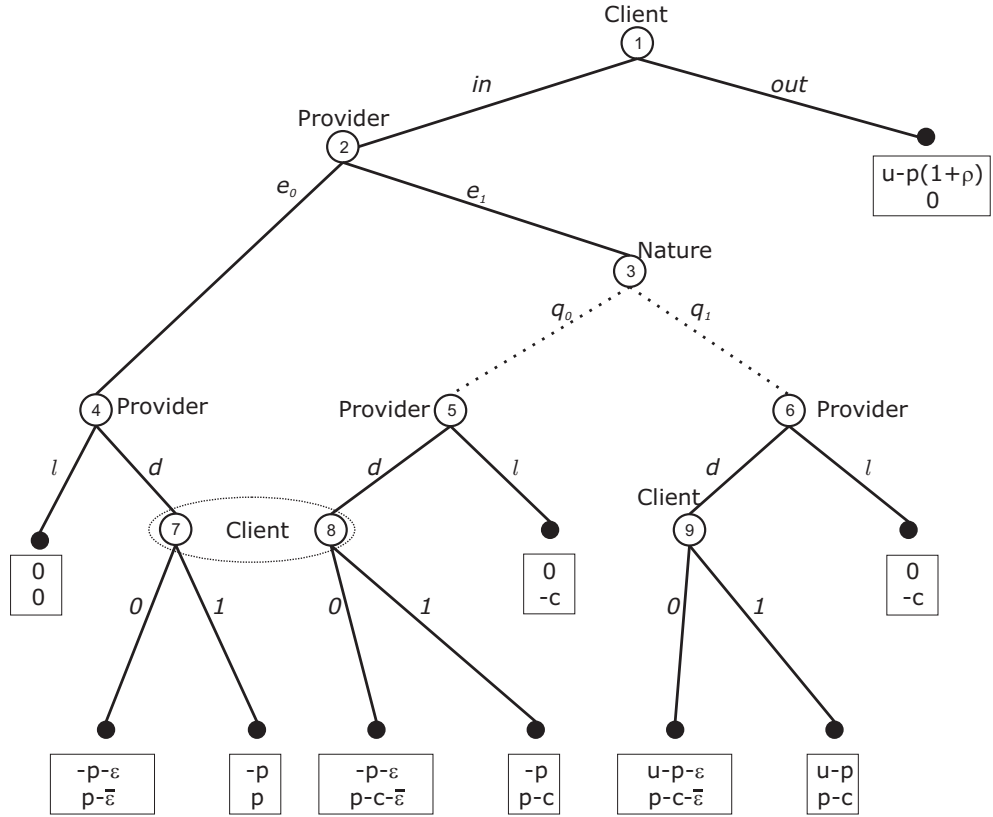}}
    \caption{The game representing one interaction. Empty circles represent decision nodes, edge
    labels represent actions, full circles represent terminal nodes and the dotted oval represents
    an information set. Payoffs are represented in
    rectangles, the top row describes the payoff of the client, the second row describes the payoff
    of the provider.}
    \label{fig:game}
\end{figure}

A pure strategy is a deterministic mapping describing an action for each of the
player's information sets. The client has three information sets in the game
$G$. The first information set is singleton and contains the node 1 at the
beginning of game when the client must decide between playing $in$ or $out$.
The second information set contains the nodes 7 and 8 (the dotted oval in
Figure $\ref{fig:game}$) where the client must decide between reporting $0$ or
$1$, given that she has received low quality, $q_0$. The third information set
is singleton and contains the node 9 where the client must decide between
reporting $0$ or $1$, given that she received high quality, $q_1$. The strategy
$in 0^{q_0}1^{q_1}$, for example, is the honest reporting strategy, specifying
that the client enters the game, reports $0$ when she receives low quality, and
reports $1$ when she receives high quality. The set of pure strategies of the
client is:
\[
    A_C = \{ out 1^{q_0}1^{q_1}, out 1^{q_0}0^{q_1}, out 0^{q_0}1^{q_1}, out
    0^{q_0}0^{q_1}, in 1^{q_0}1^{q_1}, in 1^{q_0}0^{q_1}, in 0^{q_0}1^{q_1}, in 1^{q_0}1^{q_1}
    \};
\]

Similarly, the set of pure strategies of the provider is:
\[
    A_P = \{ e_0 l, e_0 d, e_1 l^{q_0} l^{q_1}, e_1 l^{q_0} d^{q_1}, e_1 d^{q_0} l^{q_1}, e_1 d^{q_0} d^{q_1}\};
\]
where $e_1 l^{q_0} d^{q_1}$, for example, is the socially desired strategy: the
provider exerts effort at node 2, acknowledges low quality at node 5, and
delivers high quality at node 6. A pure strategy profile $s$ is a pair
$(s_C,s_P)$ where $s_C \in A_C$ and $s_P \in A_P$. If $\Delta(A)$ denotes the
set of probability distributions over the elements of $A$, $\sigma_C \in
\Delta(A_C)$ and $\sigma_P \in \Delta(A_P)$ are mixed strategies for the
client, respectively the provider, and $\sigma = (\sigma_C, \sigma_P)$ is a
mixed strategy profile.

The payoffs to the players depend on the chosen strategy profile, and on the
move of nature. Let $g(\sigma) = \big(g_C(\sigma), g_P(\sigma)\big)$ denote the
pair of expected payoffs received by the client, respectively by the provider
when playing strategy profile $\sigma$. The function $g : \Delta(A_C) \times
\Delta(A_P) \rightarrow \mathbb{R}^2$ is characterized in Table
\ref{tab:normalForm} and also describs the normal form transformation of $G$.
Besides the corresponding payments made between the client and the provider,
Table \ref{tab:normalForm} also reflects the influence of the reputation
mechanism, as further explained in Section \ref{reputation_mech}. The four
strategies of the client that involve playing $out$ at node 1 generate the same
outcomes, and therefore, have been collapsed for simplicity into a single row
of Table \ref{tab:normalForm}.

\begin{table}
\begin{center}
\scriptsize
\begin{tabular}{cc}
\begin{sideways}Provider\end{sideways} &

\begin{tabular}{c|c|c|c|c|c|}
%header
& \multicolumn{5}{c}{Client} \\
& $in 1^{q_0} 1^{q_1}$    & $in 1^{q_0} 0^{q_1}$    & $in 0^{q_0} 1^{q_1}$    & $in 0^{q_0} 0^{q_1}$ & $out$    \\
\hline

%first row - strategy e_0 l
    $e_0 l$

    &\onecell{$0$}{$d_S$}{$0$}
    &\onecell{$0$}{$d_S$}{$0$}
    &\onecell{$0$}{$d_S$}{$0$}
    &\onecell{$0$}{$d_S$}{$0$}
    &\onecell{$u - p(1+\rho)$}{}{$0$}
    \\ \hline

%second row - strategy e_0 d
    $e_0 d$

    &\onecell{$- p$}{$c_S c_B$}{$p$}
    &\onecell{$- p$}{$c_S c_B$}{$p$}
    &\onecell{$- p-\varepsilon$}{$c_S d_B$}{$p - \bar{\varepsilon}$}
    &\onecell{$- p-\varepsilon$}{$c_S d_B$}{$p - \bar{\varepsilon}$}
    &\onecell{$u - p(1+\rho)$}{}{$0$}
    \\ \hline

%third row - strategy e_1 l^{q_0} l^{q_1}
    $e_1 l^{q_0} l^{q_1}$

    &\onecell{$0$}{}{$-c$}
    &\onecell{$0$}{}{$-c$}
    &\onecell{$0$}{}{$-c$}
    &\onecell{$0$}{}{$-c$}
    &\onecell{$u - p(1+\rho)$}{}{$0$}
    \\ \hline

%forth row - strategy e_1 l^{q_0} d^{q_1}
    $e_1 l^{q_0} d^{q_1}$

    &\onecell{$\alpha (u-p)$}{}{$\alpha p -c$}
    &\onecell{$\alpha (u-p -\varepsilon)$}{}{$\alpha (p - \bar{\varepsilon}) -c$}
    &\onecell{$\alpha (u-p)$}{}{$\alpha p -c$}
    &\onecell{$\alpha (u-p -\varepsilon)$}{}{$\alpha (p - \bar{\varepsilon}) -c$}
    &\onecell{$u - p(1+\rho)$}{}{$0$}
    \\ \hline

%fifth row - strategy e_1 d^{q_0} l^{q_1}
    $e_1 d^{q_0} l^{q_1}$

    &\onecell{$-(1-\alpha) p$}{}{$(1-\alpha) p -c$}
    &\onecell{$-(1-\alpha) p$}{}{$(1-\alpha) p -c$}
    &\onecell{$-(1-\alpha)(p+\varepsilon)$}{}{$(1-\alpha)(p- \bar{\varepsilon}) -c$}
    &\onecell{$-(1-\alpha)(p+\varepsilon)$}{}{$(1-\alpha)(p- \bar{\varepsilon}) -c$}
    &\onecell{$u - p(1+\rho)$}{}{$0$}
    \\ \hline

%sixth row - strategy e_1 d^{q_0} l^{q_1}
    $e_1 d^{q_0} d^{q_1}$

    &\onecell{$\alpha u - p$}{}{$p - c$}
    &\onecell{$\alpha (u - \varepsilon) - p$}{}{$p - \alpha \bar{\varepsilon} -c$}
    &\onecell{$\alpha u - (1-\alpha)\varepsilon -p$}{}{$p - (1-\alpha)\bar{\varepsilon} -c$}
    &\onecell{$\alpha u - \varepsilon -p$}{}{ $p- \bar{\varepsilon} -c$}
    &\onecell{$u - p(1+\rho)$}{}{$0$}
    \\ \hline
\end{tabular}
\end{tabular}
\end{center}
\caption{Normal transformation of the extensive form game, $G$}
\label{tab:normalForm}
\end{table}

\subsection{The Reputation Mechanism}
\label{reputation_mech}

For every interaction, the reputation mechanism records one of the three
different signals it may receive: \emph{positive} feedback when the client
reports $1$, \emph{negative} feedback when the client reports $0$, and
\emph{neutral} feedback when the provider rolls back the transaction and
reimburses the client. In Figure \ref{fig:game} (and Table
\ref{tab:normalForm}) positive and neutral feedback do not influence the payoff
of the provider, while negative feedback imposes a punishment equivalent to
$\bar{\varepsilon}$.

Two considerations made us choose this representation. First, we associate
neutral and positive feedback with the same reward (0 in this case) because
intuitively, the acknowledgement of failure may also be regarded as ``honest''
behavior on behalf of the provider. Failures occur despite best effort, and by
acknowledging them, the provider shouldn't suffer.

However, neutral feedback may also result because the provider did not exert
effort. The lack of punishment for these instances contradicts the goal of the
reputation mechanism to encourage exertion of effort. Fortunately, the action
$e_0 l$ can be the result of rational behavior only in two circumstances, both
excusable: one, when the provider defends himself against a malicious client
that is expected to falsely report negative feedback (details in Section
\ref{evilBuyers}), and two, when the environmental noise is too big ($\alpha$
is too small) to justify exertion of effort. Neutral feedback can be used to
estimate the parameter $\alpha$, or to detect coalitions of malicious clients,
and indirectly, may influence the revenue of the provider. However, for the
simplified model presented above, positive and neutral feedback are considered
the same in terms of generated payoffs.

The second argument relates to the role of the RM to constrain the revenue of
the provider depending on the feedback of the client. There are several ways of
doing that. \citeA{Dellarocas:2005} describes two principles, and two
mechanisms that punish the provider when the clients submit negative reports.
The first, works by exclusion. After each negative report the reputation
mechanism bans the provider from the market with probability $\pi$. This
probability can be tuned such that the provider has the incentive to cooperate
almost all the time, and the market stays efficient. The second works by
changing the conditions of future trade. Every negative report triggers the
decrease of the price the next $N$ clients will pay for the service. For lower
values of $N$ the price decrease is higher, nonetheless, $N$ can take any value
in an efficient market.

Both mechanisms work because the future losses offset the momentary gain the
provider would have had by intentionally cheating on the client. Note that
these penalties are given endogenously by lost future opportunities, and
require some minimum premiums for trusted providers. When margins are not high
enough, providers do not care enough about future transactions, and will use
the present opportunity of cheating.

Another option is to use exogenous penalties for cheating. For example, the
provider may be required to buy a licence for operating in the
market\footnote{The reputation mechanism can buy and sell market licences}. The
licence is \emph{partially destroyed} by every negative feedback. Totaly
destroyed licences must be restored through a new payment, and remaining parts
can be sold if the provider quits the market. The price of the licence and the
amount that is destroyed by a negative feedback can be scaled such that
rational providers have the incentive to cooperate. Unlike the previous
solutions, this mechanism does not require minimum transaction margins as
punishments for negative feedback are directly subtracted from the upfront
deposit.

One way or another, all reputation mechanisms foster cooperation because the
provider associates value to client feedback. Let $V(R^+)$ and $V(R^-)$ be the
value of a positive, respectively a negative report. In the game in Figure
\ref{fig:game}, $V(R^+)$ is normalized to 0, and $V(R^-)$ is
$\bar{\varepsilon}$. By using this notation, we abstract away the details of
the reputation mechanism, and retain only the essential punishment associated
with negative feedback. Any reputation mechanism can be plugged in our scheme,
as long as the particular constraints (e.g., minimum margins for transactions)
are satisfied.

One last aspect to be considered is the influence of the reputation mechanism
on the future transactions of the client. If negative reports attract lower
prices, rational long-run clients might be tempted to falsely report in order
to purchase cheaper services in the future. Fortunately, some of the mechanisms
designed for single-run clients, do not influence the reporting strategy of
long-run clients. The reputation mechanism that only keeps the last $N$ reports
\cite{Dellarocas:2005} is one of them. A false negative report only influences
the next $N$ transactions of the provider; given that more than $N$ other
requests are interleaved between any two successive requests of the same
client, a dishonest reporter cannot decrease the price for her future
transactions.

The licence-based mechanism we have described above is another example. The
price of service remains unchanged, therefore reporting incentives are
unaffected. On the other hand, when negative feedback is punished by exclusion,
clients may be more reluctant to report negatively, since they also lose a
trading partner.

\subsection{Analysis of Equilibria}
\label{eqAnalysis}

The one-time game presented in Figure \ref{fig:game} has only one subgame
equilibrium where the client opts $out$. When asked to report feedback, the
client always prefers to report $1$ (reporting $0$ attracts the penalty
$\varepsilon$). Knowing this, the best strategy for the provider is to exert
low effort and deliver the service. Knowing the provider will play $e_0 d$, it
is strictly better for the client to play $out$.

The repeated game between the same client and provider may, however, have other
equilibria. Before analyzing the repeated game, let us note that every
interaction between a provider and a particular client can be strategically
isolated and considered independently. As the provider accepts all clients and
views them identically, he will maximize his expected revenue in each of the
isolated repeated games.

From now on, we will only consider the repeated interaction between the
provider and one client. This can be modeled by a \mbox{$T$-fold} repetition of
the stage game $G$, denoted $G^T$, where $T$ is finite or infinite. In this
paper we will deal with the infinite horizon case, however, the results
obtained can also be applied with minor modifications to finitely repeated
games where $T$ is large enough.

If $\hat{\delta}$ is the per period discount factor reflecting the probability
that the market ceases to exist after each round, (or the present value of
future revenues), let us denote by $\delta$ the expected discount factor in the
game $G^T$. If our client interacts with the provider on the average every $N$
rounds, $\delta = \hat{\delta}^N$.

The life-time expected payoff of the players is computed as:
\begin{small}
\begin{equation*}
\sum_{\tau =0}^T \delta^{\tau}g_i^\tau;
\end{equation*}
\end{small}
where $i \in \{C,P\}$ is the client, respectively the provider, $g_i^\tau$ is
the expected payoff obtained by player $i$ in the $\tau^{th}$ interaction, and
$\delta^{\tau}$ is the discount applied to compute the present day value of
$g_i^\tau$.

We will consider \emph{normalized} life-time expected payoffs, so that payoffs
in $G$ and $G^T$ can be expressed using the same measure:

\begin{small}
\begin{equation}
V_i = (1 - \delta) \sum_{\tau =0}^T \delta^{\tau}g_i^\tau;
\end{equation}
\end{small}

We define the \textit{average continuation payoff} for player $i$ from period
$t$ onward (and including period $t$) as:
\begin{small}
\begin{equation}
V_i^t= (1-\delta) \sum_{\tau =t}^T \delta^{\tau -t}g_i^\tau;
\end{equation}
\end{small}

The set of outcomes publicly perceived by both players after each round is:
\begin{small}
  \begin{equation*}
      Y = \{out, l, q_0 1, q_0 0, q_1 1, q_1 0\}
  \end{equation*}
\end{small}
where:
\begin{itemize}
    \item $out$ is observed when the client opts $out$,
    \item $l$ is observed when the provider acknowledges low quality and rolls back the
    transaction,
    \item $q_i \,j$ is observed when the provider delivers quality $q_i \in
    \{q_0,q_1\}$ and the client reports $j \in \{0,1\}$.
\end{itemize}
We denote by $h^t$ a specific public history of the repeated game out of the
set $H^t=(\times Y)^t$ of all possible histories up to and including period
$t$. In the repeated game, a public strategy $\sigma_i$ of player $i$ is a
sequence of maps $(\sigma_i^t)$, where $\sigma_i^t:H^{t-1} \rightarrow
\Delta(A_i)$ prescribes the (mixed) strategy to be played in round $t$, after
the public history $h^{t-1} \in H^{t-1}$. A \emph{perfect public equilibrium}
(PPE) is a profile of public strategies $\sigma = (\sigma_C,\sigma_P)$ that,
beginning at any time $t$ and given any public history $h^{t-1}$, form a Nash
equilibrium from that point on \cite{Fudenberg/Levine/Maskin:1994}.
$V_i^t(\sigma)$ is the continuation payoff to player $i$ given by the strategy
profile $\sigma$.

$G$ is a game with \emph{product structure} since any public outcome can be
expressed as a vector of two components $(y_C,y_P)$ such that the distribution
of $y_i$ depends only on the actions of player $i \in \{C,P\}$, the client,
respectively the provider. For such games, \citeA{Fudenberg/Levine/Maskin:1994}
establish a Folk Theorem proving that any feasible, individually rational
payoff profile is achievable as a PPE of $G^\infty$ when the discount factor is
close enough to 1. The set of feasible, individually rational payoff profiles
is characterized by:
\begin{itemize}
  \item the minimax payoff to the client, obtained by the option $out$:
  $\underline{V_C} = u - p(1+\rho)$;

  \item the minimax payoff to the provider, obtained when the provider plays
  $e_0 l$: $\underline{V_P} = 0$;

  \item the pareto optimal frontier (graphically presented in Figure \ref{fig:paretoOptFront})
  delimited by the payoffs given by (linear combination of)
  the strategy profiles $(in 1^{q_0} 1^{q_1}$, $e_1 l^{q_0} d^{q_1})$,
  $(in 1^{q_0} 1^{q_1}, e_1 d^{q_0} d^{q_1})$ and $(in 1^{q_0} 1^{q_1}, e_0d)$.
\end{itemize}
and contains more than one point (i.e., the payoff when the client plays $out$)
when $\alpha(u-p) > u - p(1+\rho)$ and $\alpha p -c > 0$. Both conditions
impose restrictions on the minimum margin generated by a transaction such that
the interaction is profitable. The PPE payoff profile that gives the provider
the maximum payoff is $(\underline{V_C}, \overline{V_P})$ where:
\begin{small}
  \begin{equation*}
    \overline{V_P} = \left \{
        \begin{array}{ll}
          \alpha*u -c - u + p(1+\rho) & \mbox{if $\rho \leq \frac{u(1-\alpha)}{p}$}\\
          p+ \frac{c(p\rho -u)}{\alpha u} & \mbox{if $\rho > \frac{u(1-\alpha)}{p}$}
        \end{array} \right.
  \end{equation*}
\end{small}
and $\underline{V_C}$ is defined above.

\begin{figure}
    \centerline{\includegraphics[width=0.55\columnwidth]{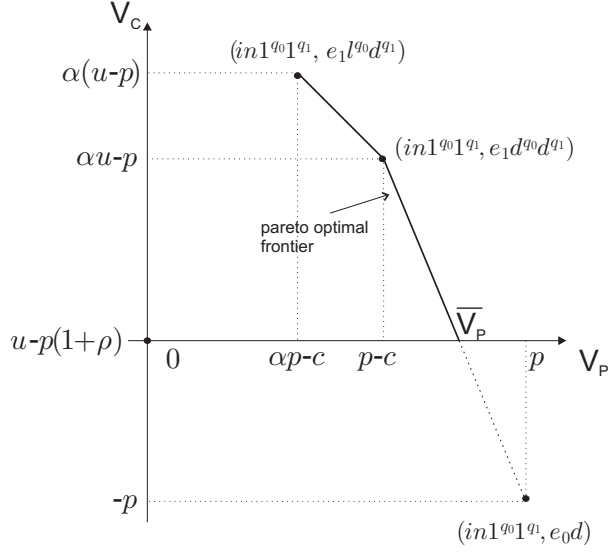}}
    \caption{The pareto-optimal frontier of the set of feasible, individually rational payoff profiles of $G$.}
    \label{fig:paretoOptFront}
\end{figure}

While completely characterizing the set of PPE payoffs for discount factors
strictly smaller than 1 is outside the scope of this paper, let us note the
following results:

First, if the discount factor is high enough (but strictly less than 1) with
respect to the profit margin obtained by the provider from one interaction,
there is at least one PPE such that the reputation mechanism records only
honest reports. Moreover, this equilibrium is pareto-optimal.

\begin{proposition}
  When $\delta > \frac{p}{p(1+\alpha) -c}$, the strategy profile:
\begin{itemize}
  \item the provider always exerts high effort, and delivers only high quality; if the
  client deviates from the equilibrium , the provider switches to $e_0 d$ for the rest of the rounds;

  \item the client always reports $1$ when asked to submit feedback; if the provider deviates,
  (i.e., she receives low quality), the client switches to $out$ for the rest
  of the rounds.
\end{itemize}
is a pareto-optimal PPE.
\label{prop:lowerBoundDelta}
\end{proposition}

\begin{proof}
  It is not profitable for the client to deviate from the equilibrium path.
  Reporting $0$ attracts the penalty $\varepsilon$ in the present round, and
  the termination of the interaction with the provider (the provider stops exerting
  effort from that round onwards).

  The provider, on the other hand, can momentarily gain by deviating to $e_1 d^{q_0} d^{q_1}$
  or $e_0 d$. A deviation to $e_1 d^{q_0} d^{q_1}$ gives an expected
  momentary gain of $p(1-\alpha)$ and an expected continuation loss of $(1-\alpha)(\alpha p
  -c)$. A deviation to $e_0 d$ brings an expected momentary gain equal to $(1-\alpha)p
  +c$ and an expected continuation loss of $\alpha p -c$. For the discount
  factor satisfying our hypothesis, both deviations are not profitable.
  The discount factor is low enough with respect to profit margins, such that
  the future revenues given by the equilibrium strategy offset the momentary
  gains obtained by deviating.

  The equilibrium payoff profile is $(V_C,V_P)=(\alpha(u-p), \alpha p -c)$,
  which is pareto-optimal and socially efficient.
\end{proof}

Second, we can prove that the client never reports negative feedback in any
pareto-optimal PPE, regardless the value of the discount factor. The
restriction to pareto-optimal is justifiable by practical reasons: assuming
that the client and the provider can somehow negotiate the equilibrium they are
going to play, it makes most sense to choose one of the pareto-optimal
equilibria.

\begin{proposition}
    The probability that the client reports negative feedback on the
    equilibrium path of any pareto-optimal PPE strategy is zero.
\label{prop:noZero}
\end{proposition}
\begin{proofsketch}
  The full proof presented in Appendix \ref{ap:noZero}
  follows the following steps. Step 1, all equilibrium payoffs
  can be expressed by adding the present round payoff to the discounted
  continuation payoff from the next round onward. Step 2, take the PPE payoff profile
  $V = (V_C,V_P)$, such that there is no other PPE payoff profile $V' = (V_C', V_P)$
  with $V_C < V_C'$. The client never reports negative feedback in the
  first round of the equilibrium that gives $V$. Step 3, the equilibrium continuation
  payoff after the first round also satisfies the conditions set for $V$. Hence, the probability that the
  client reports negative feedback on the equilibrium path that gives $V$ is 0.
  Pareto-optimal PPE payoff profiles clearly satisfy the definition of $V$,
  hence the result of the proposition.
\end{proofsketch}

The third result we want to mention here, is that there is an upper bound on
the percentage of false reports recorded by the reputation mechanism in any of
the pareto-optimal equilibria.

\begin{proposition}
  The upper bound on the percentage of false reports recorded by the reputation
  mechanism in any PPE equilibrium is:
  \begin{small}
    \begin{equation}
      \gamma \leq \left\{ \begin{array}{ll}
        \frac{(1-\alpha)(p-u) + p\rho}{p} & \mbox{if $p\rho \leq u(1-\alpha)$}; \\
        \frac{p \rho}{u} & \mbox{if $p\rho > u(1-\alpha)$}
      \end{array} \right.
      \label{eq:boundNoRep}
    \end{equation}
  \end{small}
  \label{prop:boundNoRep}
\end{proposition}

\begin{proofsketch}
  The full proof presented in Appendix \ref{ap:boundNoRep}
  builds directly on the result of Proposition \ref{prop:noZero}.
  Since clients never report negative feedback along pareto-optimal equilibria,
  the only false reports recorded by the reputation mechanism appear when the
  provider delivers low quality, and the client reports positive feedback.
  However, any PPE profile must give the client at least $\underline{V_C} = u - p(1+\rho)$,
  otherwise the client is better off by resorting to the outside option.
  Every round in which the provider deliberatively delivers low quality gives
  the client a payoff strictly smaller than $u - p(1+\rho)$. An equilibrium payoff
  greater than $\underline{V_C}$ is therefore possible only when the percentage
  of rounds where the provider delivers low quality is bounded. The same bound limits
  the percentage of false reports recorded by the reputation mechanism.
\end{proofsketch}

For a more intuitive understanding of the results presented in this section,
let us refer to the pizza delivery example detailed in Section
\ref{setting_example}. The price of a home delivered pizza is $p=1$, while at
the local restaurant the same pizza would cost $p(1+\rho) = 1.2$. The utility
of a warm pizza to the client is $u=2$, the cost of delivery is $c=0.8$ and the
probability that unexpected traffic conditions delay the delivery beyond the 30
minutes deadline (despite the best effort of the provider) is $1- \alpha =
0.01$.

The client can secure a minimax payoff of $\underline{V_C} = u - p(1+\rho) =
0.8$ by always going out to the restaurant. However, the socially desired
equilibrium happens when the client orders pizza at home, and the pizza service
exerts effort to deliver pizza in time: in this case the payoff of the client
is $V_C = \alpha (u - p) = 0.99$, while the payoff of the provider is $V_P =
\alpha p - c = 0.19$.

Proposition \ref{prop:lowerBoundDelta} gives a lower bound on the discount
factor of the pizza delivery service such that repeated clients can expect the
socially desired equilibrium. This bound is $\delta = \frac{p}{p(1+\alpha) -c}
= 0.84$; assuming that the daily discount factor of the pizza service is
$\hat{\delta} = 0.996$, the same client must order pizza at home at least once
every 6 weeks. The values of the discount factors can also be interpreted in
terms of the minimum number of rounds the client (and the provider) will likely
play the game. For example, the discount factor can be viewed as the
probability that the client (respectively the provider) will ``live'' for
another interaction in the market. It follows that the average lifetime of the
provider is at least $1/(1-\hat{\delta}) = 250$ interactions (with all
clients), while the average lifetime of the client is at least $1/(1-\delta) =
7$ interactions (with the same pizza delivery service). These are clearly
realistic numbers.

Proposition \ref{prop:boundNoRep} gives an upper bound on the percentage of
false reports that our mechanism may record in equilibrium from the clients. As
 $u(1-\alpha) = 0.02 < 0.2 = p\rho$, this limit is:
\begin{small}
\begin{equation*}
  \gamma = \frac{p\rho}{u} = 0.1;
\end{equation*}
\end{small}
It follows that at least $90\%$ of the reports recorded by our mechanism (in
any equilibrium) are correct. The false reports (false positive reports) result
from rare cases where the pizza delivery is intentionally delayed to save some
cost but clients do not complain. The false report can be justified, for
example, by the provider's threat to refuse future orders from clients that
complain. Given that late deliveries are still rare enough, clients are better
off with the home delivery than with the restaurant, hence they accept the
threat. As other options become available to the clients (e.g., competing
delivery services) the bound $\gamma$ will decrease.

Please note that the upper bound defined by Proposition \ref{prop:boundNoRep}
only depends on the outside alternative available to the provider, and is not
influenced by the punishment $\bar{\varepsilon}$ introduced by the reputation
mechanism. This happens because the revenue of a client is independent of the
interactions of other clients, and therefore, on the reputation information as
reported by other clients. Equilibrium strategies are exclusively based on the
direct experience of the client. In the following section, however, we will
refine this bound by considering that clients can build a reputation for
reporting honestly. There, the punishment $\bar{\varepsilon}$ plays an
important role.

\section{Building a Reputation for Truthful Reporting}
\label{buidingReputation}

An immediate consequence of Propositions \ref{prop:noZero} and
\ref{prop:boundNoRep} is that the provider can extract all of the surplus
created by the transactions by occasionally delivering low quality, and
convincing the clients not to report negative feedback (providers can do so by
promising sufficiently high continuation payoffs that prevent the client to
resort to the outside provider). Assuming that the provider has more ``power''
in the market, he could influence the choice of the equilibrium strategy to one
that gives him the most revenue, and holds the clients close to the minimax
payoff $\underline{V_C} = u - p(1+\rho)$ given by the outside
option.\footnote{All pareto-optimal PPE payoff profiles are also
renegotiation-proof \cite{Bernheim/Ray:1989,Farrell/Maskin:1989}. This follows
from the proof of Proposition \ref{eq:boundNoRep}: the continuation payoffs
enforcing a pareto-optimal PPE payoff profile are also pareto-optimal.
Therefore, clients falsely report positive feedback even under the more
restrictive notion of negotiation-proof equilibrium.}

However, a client who could commit to report honestly, (i.e., commit to play
the strategy $s_C^* = in 0^{q_0} 1^{q_1}$) would benefit from cooperative
trade. The provider's best response against $s^*_C$ is to play $e_1 l^{q_0}
d^{q_1}$ repeatedly, which leads the game to the socially efficient outcome.
Unfortunately the commitment to $s_C^*$ is not credible in the complete
information game, for the reasons explained in Section \ref{eqAnalysis}.

Following the results of \citeA{KMRW:1982}, \citeA{Fudenberg/Levine:1989} and
\citeA{Schmidt:1993} we know that such honest reporting commitments may become
credible in a game with incomplete information. Suppose that the provider has
incomplete information in $G^{\infty}$, and believes with non-negative
probability that he is facing a committed client that always reports the truth.
A rational client can then ``fake'' the committed client, and ``build a
reputation'' for reporting honestly. When the reputation becomes credible, the
provider will play $e_1 l^{q_0} d^{q_1}$ (the best response against $s_C^*$),
which is better for the client than the payoff she would obtain if the provider
knew she was the ``rational'' type.

As an effect of reputation building, the set of equilibrium points is reduced
to a set where the payoff to the client is higher than the payoff obtained by a
client committed to report honestly. As anticipated from Proposition
\ref{prop:boundNoRep}, a smaller set of equilibrium points also reduces the
bound of false reports recorded by the reputation mechanism. In certain cases,
this bound can be reduced to almost zero.

Formally, incomplete information can be modeled by a perturbation of the
complete information repeated game $G^{\infty}$ such that in period 0 (before
the first round of the game is played) the ``type'' of the client is drawn by
nature out of a countable set $\Theta$ according to the probability measure
$\mu$. The client's payoff now additionally depends on her type. We say that in
the perturbed game $G^{\infty}(\mu)$ the provider has incomplete information
because he is not sure about the true type of the client.

Two types from $\Theta$ have particular importance:
\begin{itemize}
  \item The ``normal'' type of the client, denoted by $\theta_0$, is the rational
  client who has the payoffs presented in Figure \ref{fig:game}.

  \item The ``commitment'' type of
  the client, denoted by $\theta^*$, always prefers to play the commitment
  strategy $s_C^*$. From a rational perspective, the commitment type client
  obtains an arbitrarily high supplementary reward for reporting the truth.
  This external reward makes the strategy $s_C^*$ the dominant strategy, and
  therefore, no commitment type client will play anything else than $s_C^*$.
\end{itemize}

In Theorem \ref{th:equilibrium} we give an upper bound $k_P$ on the number of
times the provider delivers low quality in $G^{\infty}(\mu)$, given that he
always observes the client reporting honestly.

The intuition behind this result is the following. The provider's best response
to a honest reporter is $e_1 l^{q_0} d^{q_1}$: always exert high effort, and
deliver only when the quality is high. This gives the commitment type client
her maximum attainable payoff in $G^{\infty}(\mu)$, corresponding to the
socially efficient outcome. The provider, however, would be better off by
playing against the normal type client, against whom he can obtain an expected
payoff greater than $\alpha p -c$.

The normal type client may be distinguished from a commitment type client only
in the rounds when the provider delivers low quality: the commitment type
always reports negative feedback, while the normal type might decide to report
positive feedback in order to avoid the penalty $\varepsilon$. The provider can
therefore decide to deliver low quality to the client in order to test her real
type. The question is, how many times should the provider test the true type of
the client.

Every failed test (i.e., the provider delivers low quality and the client
reports negative feedback) generates a loss of $-\bar{\varepsilon}$ to the
provider, and slightly enforces the belief that the client reports honestly.
Since the provider cannot wait infinitely for future payoffs, there must be a
time when the provider will stop testing the type of the provider, and accepts
to play the socially efficient strategy, $e_1 l^{q_0} d^{q_1}$.

The switch to the socially efficient strategy is not triggered by a revelation
of the client's type. The provider believes that the client \emph{behaves} as
if she were a commitment type, not that the client \emph{is} a commitment type.
The client may very well be a normal type who chooses to mimic the commitment
type, in the hope that she will obtain better service from the provider.
However, further trying to determine the true type of the client is too costly
for the provider. Therefore, the provider chooses to play $e_1 l^{q_0}
d^{q_1}$, which is the best response to the commitment strategy $s_C^*$.

%The following theorem defines a bound on the maximum number of times the
%provider delivers low quality in an attempt to test the true type of the
%client.

\begin{theorem}
If the provider has incomplete information in $G^{\infty}$, and assigns
positive probability to the normal and commitment type of the client
($\mu(\theta_0)>0$, $\mu_0^* = \mu(\theta^*)>0$), there is a finite upper
bound, $k_P$, on the number of times the provider delivers low quality in any
equilibrium of $G^{\infty}(\mu)$. This upper bound is:
\begin{small}
    \begin{equation}
    k_P = \left\lfloor \frac{\ln(\mu_0^*)}{\ln \left(
    \frac{\delta (\overline{V_P} - \alpha p + c) + (1-\delta)p}
            {\delta (\overline{V_P} - \alpha p + c) + (1-\delta)\bar{\varepsilon}}
    \right)} \right\rfloor
    \label{eq:k_P}
    \end{equation}
\end{small}
\label{th:equilibrium}
\end{theorem}

\begin{proof}
First, we use an important result obtained by \citeA{Fudenberg/Levine:1989}
about statistical inference (Lemma 1): If every previously delivered low
quality service was sanctioned by a negative report, the provider must expect
with increasing probability that his next low quality delivery will also be
sanctioned by negative feedback. Technically, for any $\pi < 1$, the provider
can deliver at most $n(\pi)$ low quality services (sanctioned by negative
feedback) before expecting that the $n(\pi)+1$ low quality delivery will also
be sanctioned by negative feedback with probability greater then $\pi$. This
number equals to:
\begin{small}
  \begin{equation*}
    n(\pi) = \left \lfloor \frac{\ln \mu^*}{\ln \pi} \right \rfloor;
  \end{equation*}
\end{small}

As stated earlier, this lemma does not prove that the provider will become
convinced that he is facing a commitment type client. It simply proves that
after a finite number of rounds the provider becomes convinced that the client
is playing as if she were a commitment type.

Second, if $\pi > \frac{\delta \overline{V_P}}{\delta \overline{V_P} +
(1-\delta)\bar{\varepsilon}}$ but is strictly smaller than $1$, the rational
provider does not deliver low quality (it is easy to verify that the maximum
discounted future gain does not compensate for the risk of getting a negative
feedback in the present round). By the previously mentioned lemma, it must be
that in any equilibrium, the provider delivers low quality a finite number of
times.

Third, let us analyze the round, $\bar{t}$, when the provider is about to
deliver a low quality service (play $d^{q_0}$) for the last time. If $\pi$ is
the belief of the provider that the client reports honestly in round $\bar{t}$,
his expected payoff (just before deciding to deliver the low quality service)
can be computed as follows:
\begin{itemize}
  \item with probability $\pi$ the client reports 0. Her reputation
  for reporting honestly becomes credible, so the provider plays $e_1 l^{q_0} d^{q_1}$
  in all subsequent rounds. The provider gains $p-\bar{\varepsilon}$ in the
  current round, and expects $\alpha p - c$ for the subsequent rounds;

  \item with probability $1-\pi$, the client reports $1$ and
  deviates from the commitment strategy, the provider knows he is facing a
  rational client, and can choose a continuation PPE strategy from the complete
  information game. He gains $p$ in the current round, and expects at most
  $\overline{V_P}$ in the subsequent rounds;

\end{itemize}

\begin{small}
  \begin{equation*}
    V_P \leq (1-\delta) (p - \pi \bar{\varepsilon}) +
        \delta (\pi (\alpha p -c) + (1-\pi) \overline{V_P})
  \end{equation*}
\end{small}

On the other hand, had the provider acknowledged the low quality and rolled
back the transaction (i.e., play $l^{q_0}$), his expected payoff would have
been at least:
\begin{small}
  \begin{equation*}
    V_P' \geq (1-\delta) 0 + \delta (\alpha p -c)
  \end{equation*}
\end{small}

Since the provider chooses nonetheless to play $d^{q_0}$ it must be that $V_P
\geq V_P'$ which is equivalent to:
\begin{small}
  \begin{equation}
    \pi \leq \overline{\pi} = \frac{\delta (\overline{V_P} - \alpha p + c) + (1-\delta)p}
            {\delta (\overline{V_P} - \alpha p + c) + (1-\delta)\bar{\varepsilon}}
  \label{eq:overlinePi}
  \end{equation}
\end{small}

Finally, by replacing Equation (\ref{eq:overlinePi}) in the definition of
$n(\pi)$ we obtain the upper bound on the number of times the provider delivers
low quality service to a client committed to report honestly.
\end{proof}

The existence of $k_P$ further reduces the possible equilibrium payoffs a
client can get in $G^\infty(\mu)$. Consider a rational client who receives for
the first time low quality. She has the following options:
\begin{itemize}
  \item report negative feedback and attempt to build a reputation for reporting
  honestly. Her payoff for the current round is $-p-\varepsilon$. Moreover, her worst
  case expectation for the future is that the next $k_P -1$ rounds will also give her $-p-\varepsilon$,
  followed by the commitment payoff equal to $\alpha(u-p)$:
   \begin{small}
     \begin{equation}
       V_C|0 = (1-\delta) (-p -\varepsilon) + \delta(1 -
       \delta^{k_P-1})(-p-\varepsilon) + \delta^{k_P} \alpha (u-p);
     \end{equation}
   \end{small}

  \item on the other hand, by reporting positive feedback she reveals to be a
  normal type, loses only $p$ in the current round, and expects a continuation
  payoff equal to $\hat{V}_C$ given by a PPE strategy profile of the complete
  information game $G^\infty$:
   \begin{small}
     \begin{equation}
       V_C|1 = (1-\delta) (-p) + \delta \hat{V}_C;
     \end{equation}
   \end{small}

\end{itemize}

The reputation mechanism records false reports only when clients do not have
the incentive to build a reputation for reporting honestly, and $V_C|1 >
V_C|0$; this is true for:
\begin{small}
  \begin{equation*}
        \hat{V}_C > \delta^{k_P -1} \alpha (u-p) -
        (1-\delta^{k_P-1})(p+\varepsilon) - \frac{1-\delta}{\delta}
        \varepsilon;
  \end{equation*}
\end{small}

Following the argument of Proposition \ref{prop:boundNoRep} we can obtain a
bound on the percentage of false reports recorded by the reputation mechanism
in a pareto-optimal PPE that gives the client at least $\hat{V}_C$:

\begin{small}
\begin{equation}
    \hat{\gamma} = \left\{ \begin{array}{ll}
        \frac{\alpha(u-p) -\hat{V}_C}{p} & \mbox{if $\hat{V}_C \geq \alpha u -p$}; \\
        \frac{u-p - \hat{V}_C}{u} & \mbox{if $\hat{V}_C < \alpha u -p$}
      \end{array} \right.
    \label{eq:boundWithRep}
\end{equation}
\end{small}

Of particular importance is the case when $k_P = 1$. $\hat{V}_C$ and
$\hat{\gamma}$ become:
\begin{small}
\begin{equation}
    \hat{V}_C = \alpha(u-p) - \frac{1-\delta}{\delta} \varepsilon; \quad
        \hat{\gamma} = \frac{(1-\delta)\varepsilon}{\delta p};
\end{equation}
\end{small}
so the probability of recording a false report (after the first one) can be
arbitrarily close to 0 as $\varepsilon \rightarrow 0$.

For the pizza delivery example introduced in Section \ref{setting_example},
Figure \ref{fig:k_P} plots the bound, $k_P$, defined in Theorem
\ref{th:equilibrium}, as a function of the prior belief ($\mu^*_0$) of the
provider that the client is an honest reporter. We have used a value of the
discount factor equal to $\delta = 0.95$, such that on average, every client
interacts $1/(1- \delta) = 20$ times with the same provider. The penalty for
negative feedback was taken $\bar{\varepsilon} = 2.5$. When the provider
believes that $20\%$ of the clients always report honestly, he will deliver at
most 3 times low quality. When the belief goes up to $\mu^*_0 = 40\%$ no
rational provider will deliver low quality more than once.

\begin{figure}[t]
    \centerline{\includegraphics[width=0.6\columnwidth]{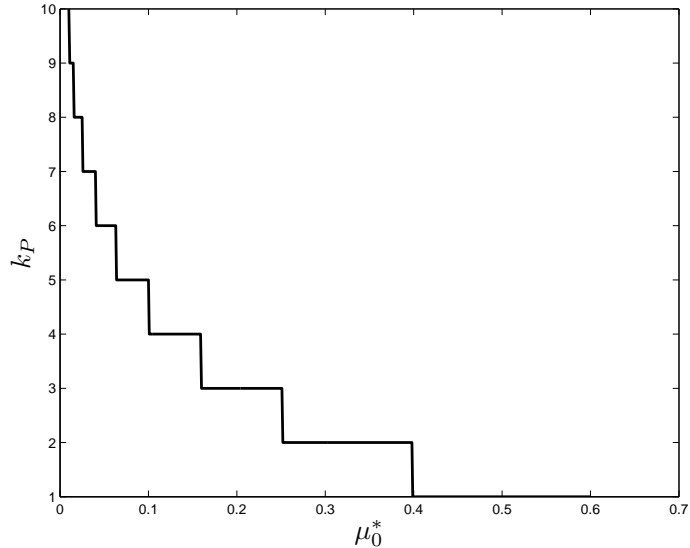}}
    \caption{The upper bound $k_P$ as a function of the prior belief $\mu_0^*$.}
    \label{fig:k_P}
\end{figure}

In Figure \ref{fig:gamma} we plot the values of the bounds $\gamma$ (Equation
(\ref{eq:boundNoRep})) and $\hat{\gamma}$ (Equation (\ref{eq:boundWithRep})) as
a function of the prior belief $\mu_0^*$. The bounds simultaneously hold,
therefore the maximum percentage of false reports recorded by the reputation
mechanism is the minimum of the two. When $\mu_0^*$ is less $0.25$, $k_P \geq
2$, $\gamma \leq \hat{\gamma}$, and the reputation effect does not
significantly reduce the worst case percentage of false reports recorded by the
mechanism. However, when $\mu_0^* \in (0.25, 0.4)$ the reputation mechanism
records (in the worst case) only half as many false reports, and as $\mu_0^* >
0.4$, the percentage of false reports drops to $0.005$. This probability can be
further decreased by decreasing the penalty $\varepsilon$. In the limit, as
$\varepsilon$ approaches 0, the reputation mechanism will register a false
report with vanishing probability.

The result of Theorem \ref{th:equilibrium} has to be interpreted as a worst
case scenario. In real markets, providers that already have a small
predisposition to cooperate will defect fewer times. Moreover, the mechanism is
self enforcing, in the sense that the more clients act as commitment types, the
higher will be the prior beliefs of the providers that new, unknown clients
will report truthfully, and therefore the easier it will be for the new clients
to act as truthful reporters.

As mentioned at the end of Section \ref{eqAnalysis}, the bound $\hat{\gamma}$
strongly depends on the punishment $\bar{\varepsilon}$ imposed by the
reputation mechanism for a negative feedback. The higher $\bar{\varepsilon}$,
the easier it is for clients to build a reputation, and therefore, the lower
the amount of false information recorded by the reputation mechanism.

\begin{figure}[t]
    \centerline{\includegraphics[width=0.6\columnwidth]{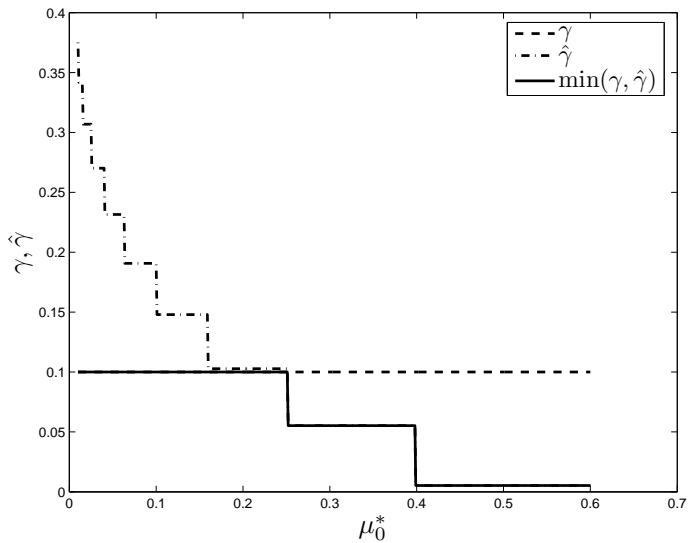}}
    \caption{The maximum probability of recording a false report as a function
    of the prior belief $\mu_0^*$.}
    \label{fig:gamma}
\end{figure}

\section{The Threat of Malicious Clients}
\label{evilBuyers}

The mechanism described so far encourages service providers to do their best
and deliver good service. The clients were assumed rational, or committed to
report honestly, and in either case, they never report negative feedback
unfairly. In this section, we investigate what happens when clients explicitly
try to ``hurt'' the providers by submitting fake negative ratings to the
reputation mechanism.

An immediate consequence of fake negative reports is that clients lose money.
However, the costs $\varepsilon$ of a negative report would probably be too
small to deter clients with separate agendas from hurting the provider.
Fortunately, the mechanism we propose naturally protects service providers from
consistent attacks initiated by malicious clients.

Formally, a \emph{malicious type} client, $\theta_\beta \in \Theta$, obtains a
supplementary (external) payoff $\beta$ for reporting negative feedback.
Obviously, $\beta$ has to be greater than the penalty $\varepsilon$, otherwise
the results of Proposition \ref{prop:noZero} would apply. In the incomplete
information game $G^\infty(\mu)$, the provider now assigns non-zero initial
probability to the belief that the client is malicious.

When only the normal type, $\theta_0$, the honest reporter type $\theta^*$ and
the malicious type $\theta_\beta$ have non-zero initial probability, the
mechanism we describe is robust against unfair negative reports. The first
false negative report exposes the client as being malicious, since neither the
normal, nor the commitment type report $0$ after receiving high quality. By
Bayes' Law, the provider's updated belief following a false negative report
must assign probability 1 to the malicious type. Although providers are not
allowed to refuse service requests, they can protect themselves against
malicious clients by playing $e_0 l$: i.e., exert low effort and reimburse the
client afterwards. The RM records neutral feedback in this case, and does not
sanction the provider. Against $e_0 l$, malicious clients are better off by
quitting the market (opt $out$), thus stopping the attack. The RM records at
most one false negative report for every malicious client, and assuming that
identity changes are difficult, providers are not vulnerable to unfair
punishments.

When other types (besides $\theta_0, \theta^*$ and $\theta_\beta$) have
non-zero initial probability, malicious clients are harder to detect. They
could masquerade client types that are normal, but accidentally misreport. It
is not rational for the provider to immediately exclude (by playing $e_0 l$)
normal clients that rarely misreport: the majority of the cooperative
transactions rewarded by positive feedback still generate positive payoffs. Let
us now consider the client type $\theta_0(\nu) \in \Theta$ that behaves exactly
like the normal type, but misreports $0$ instead of $1$ independently with
probability $\nu$. When interacting with the client type $\theta_0(\nu)$, the
provider receives the maximum number of unfair negative reports when playing
the efficient equilibrium: i.e., $e_1 l^{q_0} d^{q_1}$. In this case, the
provider's expected payoff is:
\begin{small}
\begin{equation*}
  V_P = \alpha p - c - \nu \bar{\varepsilon};
\end{equation*}
\end{small}
Since $V_P$ has to be positive (the minimax payoff of the provider is 0, given
by $e_0 l$), it must be that $\nu \leq \frac{\alpha p - c}{\bar{\varepsilon}}$.

The maximum value of $\nu$ is also a good approximation for the maximum
percentage of false negative reports the malicious type can submit to the
reputation mechanism. Any significantly higher number of harmful reports
exposes the malicious type and allows the provider to defend himself.

Note, however, that the malicious type can submit a fraction $\nu$ of false
reports only when the type $\theta_0(\nu)$ has positive prior probability. When
the provider does not believe that a normal client can make so many mistakes
(even if the percentage of false reports is still low enough to generate
positive revenues) he attributes the false reports to a malicious type, and
disengages from cooperative behavior. Therefore, one method to reduce the
impact of malicious clients is to make sure that normal clients make few or no
mistakes. Technical means (for example by providing automated tools for
formatting and submitting feedback), or improved user interfaces (that make it
easier for human users to spot reporting mistakes) will greatly limit the
percentage of mistakes made by normal clients, and therefore, also reduce the
amount of harm done by malicious clients.

One concrete method for reducing mistakes is to solicit only negative feedback
from the clients (the principle that no news is good news, also applied by
\citeA{Dellarocas:2005}). As reporting involves some conscious decision,
mistakes will be less frequent.  On the other hand, the reporting effort will
add to the penalty for a negative report, and makes it harder for normal
clients to establish a reputation for honest reporters. Alternative methods for
reducing the harm done by malicious clients (like filtering mechanisms, etc., )
as well as tighter bounds on the percentage of false reports introduced by such
clients will be further addressed in future work.

%%%%%%%%%%%%%%%%%%%%%%%%%%%%%%%%%%%%%%%%%%%%%
% DISCUSSION
%%%%%%%%%%%%%%%%%%%%%%%%%%%%%%%%%%%%%%%%%%%%%

\section{Discussion and Future Work}
\label{future_work}

Further benefits can be obtained if the clients' reputation for reporting
honestly is shared within the market. The reports submitted by a client while
interacting with other providers will change the initial beliefs of a new
provider. As we have seen in Section \ref{buidingReputation}, providers cheat
less if they a priory expect with higher probability to encounter honest
reporting clients. A client that has once built a reputation for truthfully
reporting the provider's behavior will benefit from cooperative trade during
her entire lifetime, without having to convince each provider separately.
Therefore the upper bound on the loss a client has to withstand in order to
convince a provider that she is a commitment type, becomes an upper bound on
the total loss a client has to withstand during her entire lifetime in the
market. How to effectively share the reputation of clients within the market
remains an open issue.

Correlated with this idea is the observation that clients that use our
mechanism are motivated to keep their identity. In generalized markets where
agents are encouraged to play both roles (e.g. a peer-2-peer file sharing
market where the fact that an agent acts only as ``provider'' can be
interpreted as a strong indication of ``double identity'' with the intention of
cheating) our mechanism also solves the problem signaled by
\citeA{Friedman/Resnick:2001} related to cheap online pseudonyms. The price to
pay for the new identity is the loss due to building a reputation as truthful
reporter when acting as a client.

Unlike incentive-compatible mechanism that pay reporters depending on the
feedback provided by peers, the mechanism described here is less vulnerable to
collusion. The only reason individual clients would collude is to badmouth
(i.e., artificially decrease the reputation of) a provider. However, as long as
the punishment for negative feedback is not super-linear in the number of
reports (this is usually the case), coordinating within a coalition brings no
benefits for the colluders: individual actions are just as effective as the
actions when part of a coalition. The collusion between the provider and client
can only accelerate the synchronization of strategies on one of the PPE
profiles (collusion on a non-PPE strategy profile is not stable), which is
rather desirable. The only profitable collusion can happen when competitor
providers incentivize normal clients to unfairly downrate their current
provider. Colluding clients become \emph{malicious} in this case, and the
limits on the harm they can do are presented in Section \ref{evilBuyers}.

The mechanism we describe here is not a general solution for all online
markets. In general retail e-commerce, clients don't usually interact with the
same service provider more than once. As we have showed along this paper, the
assumption of a repeated interaction is crucial for our results. Nevertheless,
we believe there are several scenarios of practical importance that do meet our
requirements (e.g., interactions that are part of a supply chain). For these,
our mechanism can be used in conjunction with other reputation mechanisms to
guarantee reliable feedback and improve the overall efficiency of the market.

Our mechanism can be further criticized for being centralized. The reputation
mechanism acts as a central authority by supervising monetary transactions,
collecting feedback and imposing penalties on the participants. However, we see
no problem in implementing the reputation mechanism as a distributed system.
Different providers can use different reputation mechanisms, or, can even
switch mechanisms given that some safeguarding measures are in place. Concrete
implementations remain to be addressed by future work.

Although we present a setting where the service always costs the same amount,
our results can be easily extended to scenarios where the provider may deliver
different kinds of services, having different prices. As long as the provider
believes that requests are randomly drawn from some distribution, the bounds
presented above can be computed using the average values of $u$, $p$ and $c$.
The constraint on the provider's belief is necessary in order to exclude some
unlikely situations where the provider cheats on a one time high value
transaction, knowing that the following interactions carry little revenue, and
therefore, cannot impose effective punishments.

In this paper, we systematically overestimate the bounds on the worst case
percentage of false reports recorded by the mechanism. The computation of tight
bounds requires a precise quantitative description of the actual set of PPE
payoffs the client and the provider can have in $G^\infty$.
\citeA{Fudenberg/Levine/Maskin:1994} and \citeA{Abreu/Pearce/Stacchetti:1990}
pose the theoretical grounds for computing the set of PPE payoffs in an
infinitely repeated game with discount factors strictly smaller than 1.
However, efficient algorithms that allow us to find this set are still an open
question. As research in this domain progresses, we expect to be able to
significantly lower the upper bounds described in Sections \ref{GTanalysis} and
\ref{buidingReputation}.

One direction of future research is to study the behavior of the above
mechanism when there is two-sided incomplete information: i.e. the client is
also uncertain about the type of the provider. A provider type of particular
importance is the ``greedy'' type who always likes to keep the client to a
continuation payoff arbitrarily close to the minimal one. In this situation we
expect to be able to find an upper bound $k_C$ on the number of rounds in which
a rational client would be willing to test the true type of the provider. The
condition $k_P < k_C$ describes the constraints on the parameters of the system
for which the reputation effect will work in the favor of the client: i.e. the
provider will give up first the ``psychological'' war and revert to a
cooperative equilibrium.

The problem of involuntary reporting mistakes briefly mentioned in Section
\ref{evilBuyers} needs further addressing. Besides false negative mistakes
(reporting $0$ instead of $1$), normal clients can also make false positive
mistakes (report $1$ instead of the intended $0$). In our present framework,
one such mistake is enough ro ruin the reputation of a normal type client to
report honestly. This is one of the reasons why we chose a sequential model
where the feedback of the client is not required if the provider acknowledges
low quality. Once the reputation of the client becomes credible, the provider
always rolls back the transactions that generate (accidentally or not) low
quality, so the client is not required to continuously defend her reputation.
Nevertheless, the consequences of reporting mistakes in the reputation building
phase must be considered in more detail. Similarly, mistakes made by the
provider, monitoring and communication errors will also influence the results
presented here.

%TO DO $\varepsilon$ and $\bar{\varepsilon}$ that change over time.

Last, but not the least, practical implementations of the mechanism we propose
must address the problem of persistent online identities. One possible attack
created by easy identity changes has been mentioned in Section
\ref{evilBuyers}: malicious buyers can continuously change identity in order to
discredit the provider. In another attack, the provider can use fake identities
to increase his revenue. When punishments for negative feedback are generated
endogenously by decreased prices in a fixed number of future transactions
\cite<e.g.,>{Dellarocas:2005}, the provider can adopt the following strategy:
he cheats on all real customers, but generates a sufficient number of fake
transactions in between two real transactions, such that the effect created by
the real negative report disappears. An easy fix to this latter attack is to
charge transaction or entrance fees. However, these measures also affect the
overall efficiency of the market, and therefore, different applications will
most likely need individual solutions.

%%%%%%%%%%%%%%%%%%%%%
%%   CONCLUSION
%%%%%%%%%%%%%%%%%%%%%

\section{Conclusions}
\label{conclusions}

Effective reputation mechanisms must provide appropriate incentives in order to
obtain honest feedback from self-interested clients. For environments
characterized by adverse selection, direct payments can explicitly reward
honest information by conditioning the amount to be paid on the information
reported by other peers. The same technique unfortunately does not work when
service providers have moral hazard, and can individually decide which requests
to satisfy. Sanctioning reputation mechanisms must therefore use other
mechanisms to obtain reliable feedback.

In this paper we describe an incentive-compatible reputation mechanism when the
clients also have a repeated presence in the market. Before asking feedback
from the clients, we allow the provider to acknowledge failures and reimburse
the price paid for service. When future transactions generate sufficient
profit, we prove that there is an equilibrium where the provider behaves as
socially desired: he always exerts effort, and reimburses clients that
occasionally receive bad service due to uncontrollable factors. Moreover, we
analyze the set of pareto-optimal equilibria of the mechanism, and establish a
limit on the maximum amount of false information recorded by the mechanism. The
bound depends both on the external alternatives available to clients and on the
ease with which they can commit to reporting the truth.

\appendix

%%%%%%%%%%%%%%%%%%%%%%%%%
\section{Proof of Proposition \ref{prop:noZero}}
\label{ap:noZero}

\textit{
    The probability that the client reports negative feedback on the
    equilibrium path of any pareto-optimal PPE strategy is zero.
}

\begin{proof}

\emph{Step 1}. Following the principle of dynamic programming
\cite{Abreu/Pearce/Stacchetti:1990}, the payoff profile $V = (V_C,V_P)$ is a
PPE of $G^\infty$, if and only if there is a strategy profile $\sigma$ in $G$,
and the continuation PPE payoffs profiles $\{W(y) | y \in Y\}$ of $G^\infty$,
such that:
\begin{itemize}
    \item $V$ is obtained by playing $\sigma$ in the current round, and a PPE
    strategy that gives $W(y)$ as a continuation payoff, where $y$ is the public outcome of the
    current round, and $Pr[y|\sigma]$ is the probability of observing $y$ after playing $\sigma$:
    \begin{small}
     \begin{equation*}
      \begin{split}
          V_C &= (1-\delta)g_C(\sigma) + \delta \Big( \sum_{y \in Y}
          Pr[y|\sigma] \cdot W_C(y) \Big); \\
          V_P &= (1-\delta)g_P(\sigma) + \delta \Big( \sum_{y \in Y}
          Pr[y|\sigma] \cdot W_P(y) \Big);
      \end{split}
     \end{equation*}
    \end{small}

    \item no player finds it profitable to deviate from $\sigma$:
    \begin{small}
      \begin{equation*}
        \begin{split}
         V_C & \geq (1-\delta)g_C \big( (\sigma_C',\sigma_P) \big) + \delta \Big( \sum_{y \in Y}
            Pr \big[y|(\sigma_C',\sigma_P)\big] \cdot W_C(y) \Big); \quad \forall \sigma_C' \neq \sigma_C \\
         V_P & \geq (1-\delta)g_P \big( (\sigma_C,\sigma_P') \big) + \delta \Big( \sum_{y \in Y}
            Pr \big[y|(\sigma_C,\sigma_P')\big] \cdot W_P(y) \Big); \quad \forall \sigma_P' \neq \sigma_P\\
        \end{split}
      \end{equation*}
    \end{small}
\end{itemize}

The strategy $\sigma$ and the payoff profiles $\{W(y) | y \in Y\}$ are said to
\emph{enforce} $V$.

\emph{Step 2.} Take the PPE payoff profile ${V} = (V_C,V_P)$, such that there
is no other PPE payoff profile $V' = (V_C', V_P)$ with $V_C < V_C'$. Let
$\sigma$ and $\{W(y) | y \in Y\}$ enforce $V$, and assume that $\sigma$ assigns
positive probability $\beta_0 = Pr[q_0 0|\sigma] > 0$ to the outcome $q_0 0$.
If $\beta_1 = Pr[q_0 1|\sigma]$ (possibly equal to 0), let us consider:
\begin{itemize}
  \item the strategy profile $\sigma' = (\sigma_C',\sigma_P)$ where $\sigma'_C$
  is obtained from $\sigma_C$ by asking the client to report $1$ instead of $0$
  when she receives low quality (i.e., $q_0$);

  \item the continuation payoffs $\{W'(y) | y \in Y\}$ such that
  $W'_i(q_0 1) = \beta_0 W_i(q_0 0) + \beta_1  W_i(q_0 1)$ and $W'_i(y \neq q_0 1)
  = W_i(y)$ for $i \in \{C,P\}$. Since, the set of correlated PPE payoff profiles
  of $G^\infty$ is convex, if $W(y)$ are PPE payoff profiles, so are $W'(y)$.
\end{itemize}

The payoff profile $(V'_C, V_P)$, $V'_C = V_C + (1-\delta)\beta_0 \varepsilon$
is a PPE equilibrium profile because it can be enforced by $\sigma'$ and
$\{W'(y) | y \in Y\}$. However, this contradicts our assumption that $V'_C <
V_C$, so $Pr[q_0 0|\sigma]$ must be 0. Following exactly the same argument, we
can prove that $Pr[q_1 0|\sigma] = 0$.

\emph{Step 3.} Taking $V$, $\sigma$ and $\{W(y) | y \in Y\}$ from step 2, we
have:
\begin{small}
 \begin{equation}
      V_C = (1-\delta)g_C(\sigma) + \delta \Big( \sum_{y \in Y}
      Pr[y|\sigma] \cdot W_C(y) \Big);
      \label{eq:step3}
 \end{equation}
\end{small}
If no other PPE payoff profile $V' = (V_C', V_P)$ can have $V'_C > V_C$, it
must be that the continuation payoffs $W(y)$ satisfy the same property. (Assume
otherwise that there is a PPE $(W_C'(y),W_P(y))$ with $W_C'(y) > W_C(y)$.
Replacing $W'_C(y)$ in (\ref{eq:step3}) we obtain $V'$ that contradicts the
hypothesis).

By continuing the recursion, we obtain that the client never reports $0$ on the
equilibrium path that enforces a payoff profile as defined in Step 2.
Pareto-optimal payoff profiles clearly enter this category, hence the result of
the proposition.
\end{proof}

%%%%%%%%%%%%%%%%%%%%%%%%

\section{Proof of Proposition \ref{prop:boundNoRep}}
\label{ap:boundNoRep}

\textit{
  The upper bound on the percentage of false reports recorded by the reputation
  mechanism in any PPE equilibrium is:
  \begin{small}
    \begin{equation*}
      \gamma \leq \left\{ \begin{array}{ll}
        \frac{(1-\alpha)(p-u) + p\rho}{p} & \mbox{if $p\rho \leq u(1-\alpha)$}; \\
        \frac{p \rho}{u} & \mbox{if $p\rho > u(1-\alpha)$}
      \end{array} \right.
    \end{equation*}
  \end{small}
}
\begin{proof}
  Since clients never report negative feedback along pareto-optimal equilibria,
  the only false reports recorded by the reputation mechanism appear when the
  provider delivers low quality, and the client reports positive feedback. Let
  $\sigma = (\sigma_C, \sigma_P)$ be a pareto-optimal PPE strategy profile.
  $\sigma$ induces a probability distribution over public histories and,
  therefore, over expected outcomes in each of the following transactions. Let
  $\mu_t$ be the probability distribution induced by $\sigma$ over the outcomes in round
  $t$. $\mu_t(q_0 0) = \mu_t(q_1 0) = 0$ as proven by Proposition
  \ref{prop:noZero}.
  The payoff received by the client when playing $\sigma$ is therefore:
  \begin{small}
    \begin{equation*}
      V_C(\sigma) \leq (1-\delta) \sum_{t=0} ^\infty \delta^t \Big( \mu_t(q_0 1)
      (-p) + \mu_t(q_1 1) (u-p) + \mu_t(l) 0 + \mu_t(out) (u-p-p\rho) \Big);
    \end{equation*}
  \end{small}
  where $\mu_t(q_0 1) + \mu_t(q_1 1) + \mu_t(l) + \mu_t(out) = 1$ and
  $\mu_t(q_0 1) + \mu_t(l) \geq (1-\alpha) \mu_t(q_1 1) /
  \alpha$, because the probability of $q_0$ is at least $(1-\alpha)/\alpha$ times
  the probability of $q_1$.

  When the
  discount factor, $\delta$, is the probability that the repeated interaction
  will stop after each transaction, the expected probability of the outcome
  $q_0 1$ is:
  \begin{small}
    \begin{equation*}
      \gamma = (1-\delta) \sum_{t=0}^\infty \delta^t \mu_t(q_0 1);
    \end{equation*}
  \end{small}

  Since any PPE profile must give the client at least $\underline{V_C} = u - p(1+\rho)$,
  (otherwise the client is better off by resorting to the outside option),
  $V_C(\sigma) \geq \underline{V_C}$. By replacing the expression of
  $V_C(\sigma)$, and taking into account the constraints on the probability of $q_1$ we obtain:
  \begin{small}
    \begin{equation*}
      \gamma (-p) + (u-p) \cdot \min\big(1-\gamma, \alpha\big)  \leq \underline{V_C};
    \end{equation*}
  \end{small}
  \begin{small}
    \begin{equation*}
      \gamma \leq \left\{ \begin{array}{ll}
        \frac{(1-\alpha)(p-u) + p\rho}{p} & \mbox{if $p\rho \leq u(1-\alpha)$}; \\
        \frac{p \rho}{u} & \mbox{if $p\rho > u(1-\alpha)$}
      \end{array} \right.
    \end{equation*}
  \end{small}
\end{proof}

%\bibliography{../../reff}
%\bibliographystyle{theapa}

\end{document}